\documentclass{article}

\usepackage{arxiv}
\usepackage{nicefrac}
\usepackage[utf8]{inputenc} % allow utf-8 input
\usepackage[T1]{fontenc}  
\usepackage{microtype} 
\usepackage{url} 
\usepackage{amsfonts}
\usepackage{booktabs} % For formal tables
\usepackage[utf8]{inputenc}
\usepackage{caption}
\usepackage{subcaption}
\usepackage{csquotes}
\usepackage{xcolor}
\usepackage{graphicx}
\usepackage{amsmath,amssymb,amsthm}
\usepackage{algorithm}
\usepackage[noend]{algorithmic}
\usepackage{float}
\usepackage{soul}
\usepackage{tikz}
\usepackage{pgfplots}
\usepackage{multirow}
\usepackage[inline]{enumitem}

\sethlcolor{lightgray}

\newcommand{\COMM}[1]{\hfill\textcolor{gray!80}{// #1}}
\DeclareMathOperator*{\argmax}{arg\,max}

\newtheorem{theorem}{Theorem}%[chapter]
\newtheorem{corollary}{Corollary}%[chapter]
\pgfplotsset{compat=1.17}

\author{ Adel Nikfarjam \\
Optimisation and Logistics\\School of Computer Science\\The University of Adelaide\\
  \texttt{adel.nikfarjam@adelaide.edu.au} \\
  \And
  Anh Viet Do\\
Optimisation and Logistics\\School of Computer Science\\The University of Adelaide\\
  \texttt{vietanh.do@adelaide.edu.au} \\
  \And
   Frank Neumann \\
Optimisation and Logistics\\School of Computer Science\\The University of Adelaide\\
  \texttt{frank.neumann@adelaide.edu.au} \\}

\begin{document}
\title{Analysis of Quality Diversity Algorithms for the Knapsack Problem}
\maketitle
% \titlenote{Produces the permission block, and

%%% The submitted version for review should be ANONYMOUS
% \author{Adel Nikfarjam}
% %\email{adel.nikfarjam@adelaide.edu.au}
% \affiliation{%
%   \institution{Optimisation and Logistics\\School of Computer Science\\The University of Adelaide}
%   \city{Adelaide}
%   \country{Australia}
% }

% \author{Aneta Neumann}
% %\email{aneta.neumann@adelaide.edu.au}
% \affiliation{%
%   \institution{Optimisation and Logistics\\School of Computer Science\\The University of Adelaide}
%   \city{Adelaide}
%   \country{Australia}
% }
% \author{Frank Neumann}
% %\email{frank.neumann@adelaide.edu.au}
% \affiliation{%
%     \institution{Optimisation and Logistics\\School of Computer Science\\The University of Adelaide}
%   \city{Adelaide}
%   \country{Australia}
% }


\begin{abstract}
%\frank{draft of abstract. Please check and adjust!}
Quality diversity (QD) algorithms have been shown to be very successful when dealing with problems in areas such as robotics, games and combinatorial optimization. They aim to maximize the quality of solutions for different regions of the so-called behavioural space of the underlying problem. In this paper, we apply the QD paradigm to simulate dynamic programming behaviours on knapsack problem, and provide a first runtime analysis of QD algorithms. We show that they are able to compute an optimal solution within expected pseudo-polynomial time, and reveal parameter settings that lead to a fully polynomial randomised approximation scheme (FPRAS). Our experimental investigations evaluate the different approaches on classical benchmark sets in terms of solutions constructed in the behavioural space as well as the runtime needed to obtain an optimal solution.
\end{abstract}

%\frank{take out usepackage{times} if still fits into 12 pages, I put it in to give you more space}

\section{Introduction}

%\frank{We should cite the papers on EAs and dynamic programming:}
%\cite{DBLP:journals/tcs/DoerrENTT11,DBLP:conf/evoW/Theile09,DBLP:conf/foga/Horoba09}
%\frank{We should say clearly in the introduction that to our knowledge this is the first runtime %analysis of QD algorithms.}
Computing diverse sets of high quality solutions has recently gained significant interest in the evolutionary computation community under the terms Evolutionary Diversity Optimisation~(EDO) and Quality Diversity~(QD). With this paper, we contribute to the theoretical understanding of such approaches algorithms by providing a first runtime analysis of QD algorithms. We provide rigorous results for the classical knapsack problem and carry out additional experimental investigations on the search process in the behavioral space.

Diversity is traditionally seen as a mechanism to explore niches in a fitness landscape and prevent premature convergence during evolutionary searches. On the other hand, the aim of EDO is to explicitly maximise the structural diversity of a set of solutions, which usually have to fulfill some quality criteria. The concept was first introduced in \cite{ulrich2011maximizing} in a continuous domain. Later, EDO has been adopted to evolve a set of images~\cite{alexander2017evolution} and benchmark instances for traveling salesperson problem~(TSP) \cite{doi:10.1162/evcoa00274}. The star-discrepancy and the indicators from multi-objective evolutionary algorithms have been incorporated in EDO for the same purpose as the previous studies in \cite{neumann2018discrepancy} and \cite{neumann2019evolutionary}, respectively. More recently, EDO has been investigated in context of computing a diverse set of solutions for several combinatorial problems such as TSP in \cite{viet2020evolving,NikfarjamBN021a,NikfarjamB0N21b}, the quadratic assignment problem~\cite{DoGN021}, the minimum spanning tree problem~\cite{Bossek021tree}, the knapsack problem~\cite{BossekN021KP}, the optimisation of monotone sub-modular functions~\cite{NeumannB021}, and traveling thief problem \cite{NikTTPEDO}. 

On the other hand, QD explores a predefined behavioural space to find niches. It recently has gained increasing attention among the researchers in evolutionary computation. The optimisation paradigm first emerged in the form of novelty search, in which the goal is to find solutions with unique behaviours aside from the quality of solutions \cite{LehmanS11}. Later, a mechanism is introduced in \cite{CullyM13} to solely retain best-performing solutions while exploring new behaviours. An algorithm, named MAP-Elite is introduced in \cite{CluneML13} to plot the distribution of high-performing solutions in a behavioural space. MAP-Elite is shown efficient in developing behavioural repertoire. QD was coined as a term, and defined as a concept in \cite{PughSSS15,PughSS16}. The paradigm has been widely applied in the context of robotic and gaming \cite{RakicevicCK21,ZardiniZZIF21,SteckelS21,FontaineTNH20,FontaineLKMTHN21}. More recently, QD has been adopted for a multi-component combinatorial problem, namely traveling thief problem \cite{NikfarjamMap}. Bossek and Neumann \cite{JakobQD} generated diverse sets of TSP instances by the use of QD. To the best of our knowledge, the use of QD in solving a combinatorial optimisation problem is limited to an empirical study \cite{NikfarjamMap}. Although the QD-based algorithm has been shown to yield very decent results, theoretical understandings of its performance have not yet been established.

In this work, we contribute to this line of research by theoretically and empirically studying QD for the the knapsack problem (KP), with a focus on connections between populating behavioural spaces and constructing solutions in dynamic programming (DP) manner. The use of evolutionary algorithms building populations of specific structure to carry out dynamic programming  has been studied in~\cite{DBLP:journals/tcs/DoerrENTT11,DBLP:conf/evoW/Theile09,DBLP:conf/foga/Horoba09}. 
We consider a more natural way of enabling dynamic programming behavior by using QD algorithms with appropriately defined behavioural spaces.
To this end, we define two behavioural spaces based on weights, profits and the subset of the first $i$ items, as inspired by dynamic programming (DP) \cite{DBLP:Toth80} and the classic fully polynomial-time approximation scheme (FPTAS) \cite{Vijay2001}. Here, the scaling factor used in the FPTAS adjusts the niche size along the weight/profit dimension. We formulate two simple mutation-only algorithms based on MAP-Elite to populate these spaces. We show that both algorithms mimic DP and find an optimum within pseudo-polynomial expected runtime. Moreover, we show that in the profit-based space, the algorithm can be made into a fully polynomial-time randomised approximation scheme (FPRAS) with an appropriate choice of the scaling value. Our experimental investigation on various instances suggests that these algorithms significantly outperforms $(1+1)$EA and $(\mu+1)$EA, especially in hard cases. With this, we demonstrate the ability of QD-based mechanisms to imitate DP-like behaviours in KP, and thus its potential value in black-box optimisers for problem with recursive subproblem structures.

The remainder of the paper is structured as follows. We formally define the knapsack problem, the behavioural spaces, and the algorithms in \ref{sec:prob}. Next, we provide a runtime analysis for the algorithms in \ref{sec:analysis}. In Section \ref{sec:experiments}, we examine the distribution of high-quality knapsack solutions in the behavioural spaces and compare QD-based algorithms to other EAs. Finally, we finish with some concluding remarks.

\section{Quality-Diversity for the knapsack problem}
\label{sec:prob}
The knapsack problem is defined on a set of items $I$, where $|I|=n$ and each item $i$ corresponds to a weight $w_i$ and a profit $p_i$. Here, the goal is to find a selection of item $x = (x_1, x_2, \ldots , x_n)$ that maximise the profit while the weight of selected items is constrained to a capacity $C$. Here, $x$ is the characteristic vector of the selection of items. Technically, KP is a binary linear programming problem: let $w=(w_1,\ldots,w_n)$ and $p=(p_1,\ldots,p_n)$, find
%\begin{align*}
$
\argmax_{x\in\{0,1\}^n}\left\lbrace p^Tx\mid w^Tx\leq C\right\rbrace.$
%\end{align*}
We assume that all items have weights in $(0,C]$, since any item violating this can be removed from the problem instance.

% \subsection{Algorithms}

\begin{figure}[t]
\centering
% \begin{tikzpicture}
% \node (v-w) at (-2.5,-0.3) {\scriptsize{{$v(x)-w(x)$~(1)}}};
% \node (v-q) at (3.2,-0.3) {\scriptsize{{$v(x)-q(x)$~(2)}}};
% \end{tikzpicture}
\begin{subfigure}{.4\linewidth}
\centering
\includegraphics[width=0.8\columnwidth]{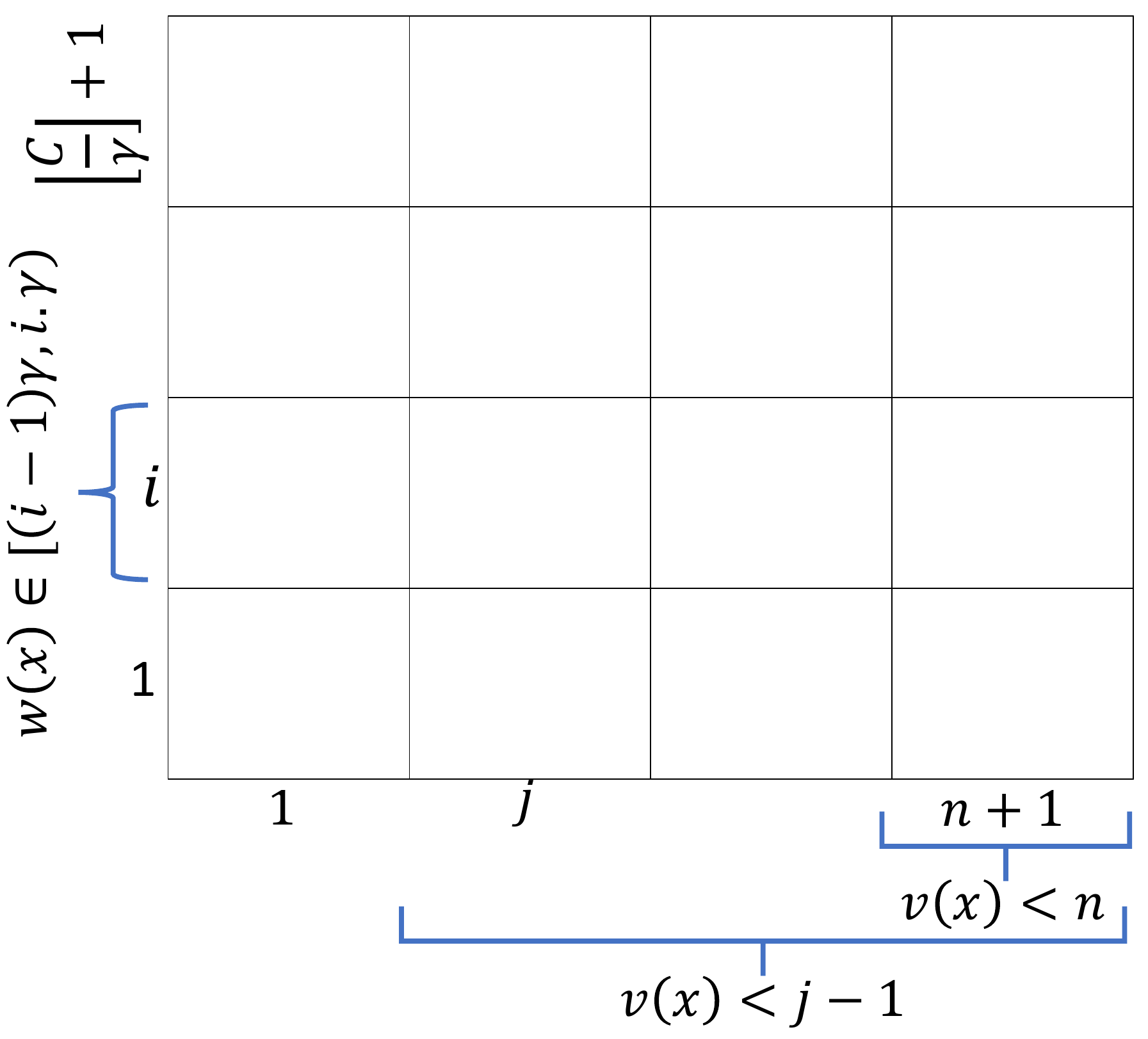}
\caption{Weight-based}\vspace{-5pt}
\label{fig:example_map1}
\end{subfigure}
\begin{subfigure}{.4\linewidth}
\centering
\includegraphics[width=0.8\columnwidth]{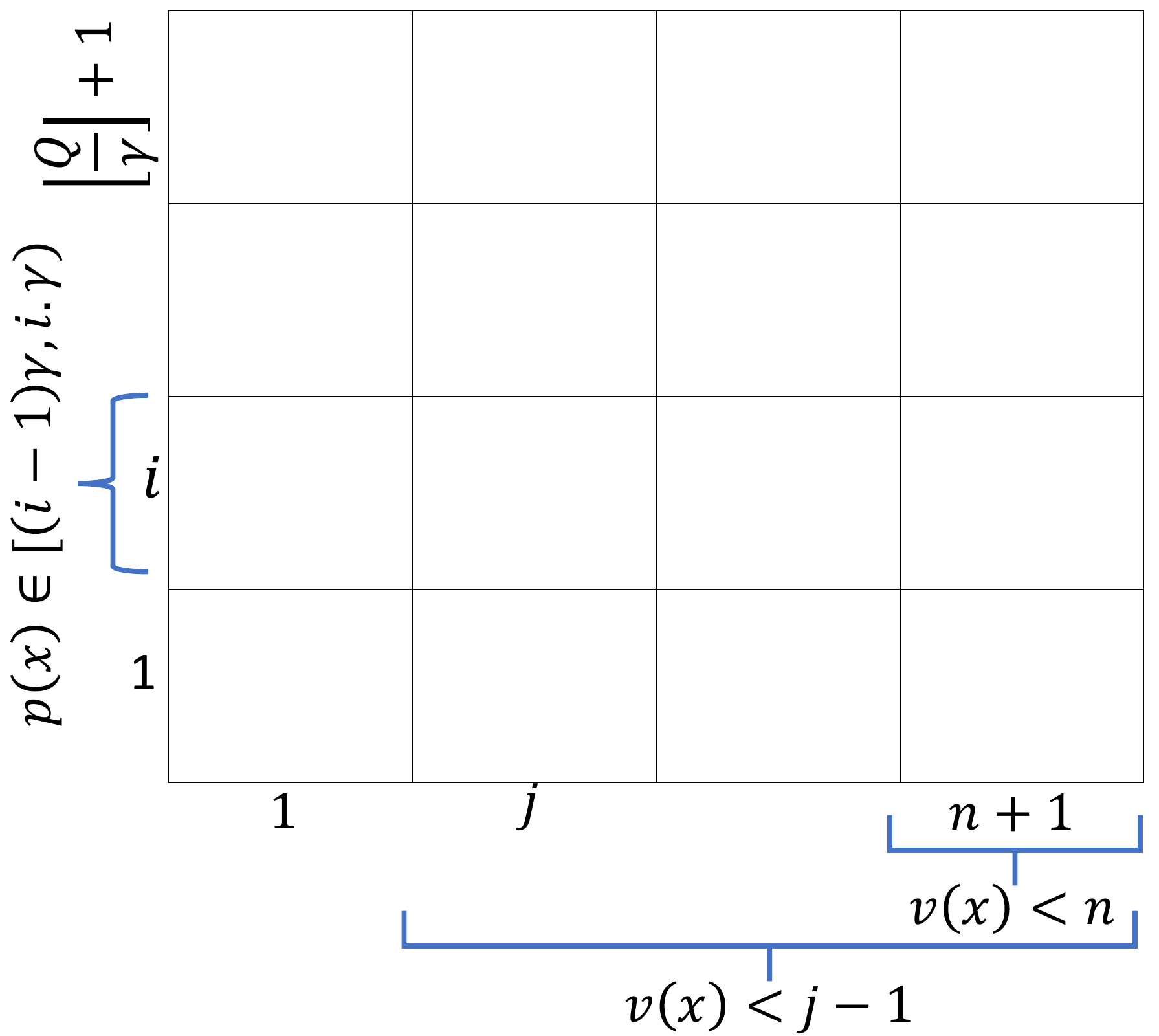}
\caption{Profit-based}\vspace{-5pt}
\label{fig:example_map2}
\end{subfigure}
\caption{The representation of the empty maps in the behavioral spaces.}\vspace{-10pt}
\label{fig:example_map}
\end{figure}
In this section, we introduce two MAP-Elite based algorithms exploring two different behavioral spaces. To determine behaviour of a solution in a particular space, a behaviour descriptor~(BD) is required. MAP-Elite is an EA, where a solution competes with other solutions with a similar BD. MAP-Elites discretizes a behavioural space into a grid to define the similarity and acceptable tolerance of difference in two descriptors. Each cell in the grid corresponds with a BD type, and only best solution with that particular BD is kept in the cell.

For KP, we formulate the behavioral spaces based on the two ways in which the classic dynamic programming approach is implemented \cite{DBLP:Toth80}, i.e. profit-based and weight-based sub-problem divisions. Let $v(x)$ be the function returning the index of the last item in solution $x$: $v(x)=\max_i\{i\mid x_i=1\}$.

\subsection{Weight-based space}
\label{SPC:One}
For the weight-based approach, $w(x)$ and $v(x)$ serve as the BD, where $w(x) = w^Tx$.
Figure \ref{fig:example_map1} outlines an empty map in the weight-based behavioural space. To exclude infeasible solutions, the weight dimension is restricted to $[0,C]$. As depicted, the behavioural space consists of $\left(\left\lfloor C/\gamma\right\rfloor+1\right) \times \left(n+1\right)$ cells, in which cell $(i,j)$ includes the best solution $x$ (i.e. maximizing $p(x) = p^Tx$) with $v(x)=j-1$ and $w(x) \in [(i-1)\gamma, i\gamma)$. Here, $\gamma$ is a factor to determine the size of each cell. The algorithm is initiated with a zero string $0^n$. Having a parent is selected uniformly at random from the population, we generate a single offspring  by standard flip mutation. If $w(x) \leq C$, we find the cell corresponding with the solution BD. We check the cells one by one. If the cells are empty, $x$ is store in the cell; otherwise, the solution with highest profit remains in the corresponding cell. These steps are continued until a termination criterion is met.

\subsection{Profit-based space}
\label{SPC:Two}
For the profit-based approach, $p(x)$ and $v(x)$ serve as the BD. Figure \ref{fig:example_map2} depicts the profit-based behavioural space with $\left(\left\lfloor Q/\gamma\right\rfloor+1\right) \times \left(n+1\right)$ cells where $Q = \sum_{i\in I}p_i$. Here, the selection in each cell minimizes the weight, and cell $(i,j)$ includes a solution $x$ with $v(x)=j-1$ and $p(x) \in [(i-1)\gamma, i\gamma)$. Otherwise, the parent selection and the operator are the same as in weight-based MAP-Elite. After generating the offspring, we determine the cell associating with the BD ($(v(x), p(x))$). If the cell are empty the solution is stored in the cells; otherwise, the solution with the lower weight $w(x)$ is kept in the cell. The steps are continued until a termination criterion is met.
\subsection{DP-based filtering scheme}
\label{sec:filtering}
In classical MAP-Elites, the competition between solutions is confined within each cell. However, in this context, the mapping from solution space to behaviour space is transparent enough in both cases that a dominance relation between solutions in different cells can be determined; a property exploited by the DP approach. Therefore, in order to reduce the population size and speed up the search for the optimum, we incorporate a filtering scheme that forms the core of the DP approach. Given solutions $x_1$ and $x_2$ with $v(x_1) \geq v(x_2)$ and $w(x_1) = w(x_2)$; then, $x_1$ dominates $x_2$ in the weight-based space if $p(x_1) > p(x_2)$. To filter out the dominated solutions, we relax the restriction that each BD corresponds to only one cell and redefine acceptable solutions for Cell $(i,j)$ in the weight-based space: $v(x) \leq j-1$ and $w(x) \in [(i-1)\gamma, i\gamma)$. This means a particular BD is acceptable for multiple cells, and MAP-Elite algorithms must check all the cells accepting the offspring. Algorithm \ref{alg:QD_weight} outlines the MAP-Elite algorithm exploring this space; this is referred to as weight-based MAP-Elites.
\begin{algorithm}[t!]
\begin{algorithmic}[1]
% \begin{footnotesize}
\REQUIRE{weights $\{w_i\}_{i=1}^n$,  $C$, profits $\{p_i\}_{i=1}^n$, $\gamma$}
\STATE $P \gets \{0^n\}$\COMM{$P$ is indexed from 1, $0^n$ is an all-zero string}
\STATE $A \gets 0_{n+1 \times \lfloor C/\gamma\rfloor +1}$\COMM{$0_{n+1 \times \lfloor C/\gamma\rfloor+1}$ is an all-zero matrix}
\STATE $B \gets 0$,
% \STATE $It \gets 0$
% \STATE $\{w'_i\}_{i=1}^n \gets \{w_i/\gamma\}_{i=1}^n$
\WHILE{Termination criteria are not met}
%\STATE $It \gets It+1$
\STATE $i\gets Uniform(\{1,\ldots,|P|\})$
\STATE Get $x$ from flipping each bit in $P(i)$ independently with probability $1/n$
\IF{$w(x) \leq C$}
\STATE $W' \gets \lfloor w(x)/\gamma \rfloor +1$
\IF{$A_{v(x) + 1, W'} = 0$}
\STATE $P\gets P\cup\{x\}$\COMM{$x$ is indexed last in $P$}
\STATE $A_{v(x) + 1, W'}\gets|P| $
\ELSIF{$p(x) > p(P(A_{v(x) + 1, W'}))$}
\STATE $P(A_{v(x) + 1, W'}) \gets x$
\ENDIF
\FOR[DP-based filtering scheme]{$j$ from $v(x) + 2$ to $n+1$}
\IF{$A_{j , W'} = 0$ Or $p(x) > p(P(A_{j, W'}))$}
\STATE $A_{j, W'} \gets A_{v(x) + 1, W'}$
\ENDIF
\ENDFOR
\IF{$p(x) > B$}
\STATE $B \gets p(x)$
\ENDIF
\ENDIF
\ENDWHILE
\RETURN{B}
% \end{footnotesize}
\end{algorithmic}
\caption{weight-based MAP-Elites}
\label{alg:QD_weight}
\end{algorithm}

The same scheme can be applied to the profit-based space, where cell $(i, j)$ accepts solution $x$ with $v(x) \leq j-1$ and $p(x) \in [(i-1)\gamma, i\gamma)$. In this case, the dominance relation is formulated to minimise weight. Algorithm \ref{alg:QD_profit} sketches the profit-based MAP-Elites.
\begin{algorithm}[t!]
\begin{algorithmic}[1]
% \begin{footnotesize}
\REQUIRE{Weights $\{w_i\}_{i=1}^n$, $C$, profits $\{p_i\}_{i=1}^n$, $\gamma$} 

%\STATE Initialise the population $P$ with $\mu$ copies of the optimal tour for the TSP.
\STATE $P \gets \{0^n\}$\COMM{$P$ is indexed from 1, $0^n$ is an all-zero string}
\STATE $A \gets 0_{n+1 \times \sum_{i=1}^n p_i+1}$\COMM{$0_{n+1 \times \lfloor C/\gamma\rfloor+1}$ is an all-zero matrix}
\STATE $B \gets 0$,
% \STATE $It \gets 0$
% \STATE $\{w'_i\}_{i=1}^n \gets \{w_i/\gamma\}_{i=1}^n$
\WHILE{Termination criteria are not met}
%\STATE $It \gets It+1$
\STATE $i\gets Uniform(\{1,\ldots,|P| \})$
\STATE Get $x$ from flipping each bit in $P(i)$ independently with probability $1/n$
\STATE $G \gets \lfloor p(x)/\gamma \rfloor +1$
\IF{$A_{v(x) + 1, G} = 0$}
\STATE $P\gets P\cup\{x\}$\COMM{$x$ is indexed last in $P$}
\STATE $A_{v(x) + 1, G}\gets|P| $
\ELSIF{$w(x) < w(P(A_{v(x) + 1, G}))$}
\STATE $P(A_{v(x) + 1, G}) \gets x$
\ENDIF
\FOR[DP-based filtering scheme]{$j$ from $v(x) + 2$ to $n+1$}
\IF{$A_{j , G} = 0$ Or $w(x) < w(P(A_{j, G}))$}
\STATE $A_{j, G} \gets A_{v(x) + 1, G}$
\ENDIF
\ENDFOR
\IF{$w(x) \leq C$}
\IF{$p(x) > B$}
\STATE $B \gets p(x)$
\ENDIF
\ENDIF
\ENDWHILE
\RETURN{B}
% \end{footnotesize}
\end{algorithmic}
\caption{profit-based MAP-Elites}
\label{alg:QD_profit}
\end{algorithm}

\section{Theoretical analysis} 
\label{sec:analysis}
In this section, we give some runtime results for Algorithm \ref{alg:QD_weight} and \ref{alg:QD_profit} based on expected time, as typically done in runtime analyses. Here, we use ``time'' as a shorthand for ``number of fitness evaluations'', which in this case equals the number of generated solutions during a run of the algorithm.
We define $a\wedge b$ and $a\vee b$ to be the bit-wise AND and bit-wise OR, respectively, between two equal length bit-strings $a$ and $b$. Also, we denote k-length all-zero and all-one bit-strings by $0^k$ and $1^k$, respectively. For convenience, we denote the $k$-size prefix of $a\in\{0,1\}^n$ with $a^{(k)}=a\wedge 1^k0^{n-k}$, and the $k$-size suffix with $a_{(k)}=a\wedge 0^{n-k}1^{k}$ .

It is important to note that in all our proofs, we consider solution $y$ replacing solution $x$ during a run to imply $v(y)\leq v(x)$. Since this holds regardless of whether filtering scheme outlined in Section \ref{sec:filtering} is used, our results should apply to both cases, as we use the largest possible upper bound of population size. Note that this filtering scheme may not reduce the population size in some cases.

We first show that with $\gamma=1$ (no scaling), Algorithm \ref{alg:QD_weight} ensures that prefixes of optimal solutions remain in the population throughout the run, and that these increase in sizes within a pseudo-polynomial expected time. For this result, we assume all weights are integers.

\begin{theorem}\label{theorem:weight_runtime}
Given $\gamma=1$ and $k\in[0,n]$, within expected time $e(C+1)n^2k$, Algorithm \ref{alg:QD_weight} achieves a population $P$ such that for any $j\in[0,k]$, there is an optimal solution $x^*$ where $x^{*(j)}\in P$.
\end{theorem}
\begin{proof}
Let $P_t$ be the population at iteration $t\geq0$, $S$ be the set of optimal solutions, $S_j=\{s^{(j)}\mid s\in S\}$, $X_t=\max\{h\mid \forall j\in[0,h],S_j\cap P_t\neq\emptyset\}$, and $H(x,y)$ be the Hamming distance between $x$ and $y$, we have $S_n=S$. We see that for any $j\in[0,X_t]$, any $x\in S_j\cap P_t$ must be in $P_{>t}$, since otherwise, let $y$ be the solution replacing it, and $y^*=y\vee x^*_{(n-j)}$ for any $x^*\in S$ where $x=x^{*(j)}$, we would have $p(y^*)-p(x^*)=p(y)-p(x)>0$ and $w(y^*)=w(x^*)\leq B$, a contradiction. Additionally, if $x\in S_i\cap S_j$ for any $0\leq i<j\leq n$, then $x\in\bigcap_{h=i}^jS_h$. Thus, if $X_t<n$, then $S_{X_t}\cap S\cap P_t=\emptyset$, so for all $x\in S_{X_t}\cap P_t$, there is $y\in S_{>X_t}$ such that $H(x,y)=1$. We can then imply from the algorithm's behaviour that for any $j\in[0,n-1]$, $Pr[X_{t+1}<j\mid X_{t}=j]=0$ and

\[Pr[X_{t+1}>j\mid X_{t}=j]\geq\frac{1}{n}\left(1-\frac{1}{n}\right)^{n-1}\frac{|S_{X_t}\cap P_t|}{|P_t|}\geq\frac{1}{en\max_h|P_h|}.\]

Let $T$ be the minimum integer such that $X_{t+T}>X_t$, then the expected waiting time in a binomial process gives $E[T\mid X_t<j]\leq en\max_h|P_h|$ for any $j\in[1,n]$. Let $T_k$ be the minimum integer such that $X_{T_k}\geq k$, we have for any $k\in[0,n]$, $E[T_k]\leq\sum_{i=1}^kE[T\mid X_t<i]\leq en\max_h|P_h|k$, given that $0^n\in S_0\cap P_0$. Applying the bound $\max_h|P_h|\leq (C+1)n$ yields the claim.
\end{proof}

We remark that with $\gamma>1$, Algorithm \ref{alg:QD_weight} may fail to maintain prefixes of optimal solutions during a run, due to rounding error. That is, assuming there is $x=x^{*(j)}\in P_t$ at step $t$ and for some $j\in[0,n]$ and optimal solution $x^*$, a solution $y$ may replace $x$ if $p(y)>p(x)$ and $w(y)<w(x)+\gamma$. It is possible that $y^*=y\vee x^*_{(n-j)}$ is infeasible (i.e. when $C<w(x^*)+\gamma$), in which case the algorithm may need to ``fix'' $y$ with multiple bit-flips in one step.
The expected runtime till optimality can be derived directly from Theorem \ref{theorem:weight_runtime} by setting $k=n$.
\begin{corollary}
Algorithm \ref{alg:QD_weight}, run with $\gamma=1$, finds an optimum within expected time $e(C+1)n^3$.
\end{corollary}

Using the notation $Q=\sum_{i=1}^np_i$, we have the following result for Algorithm \ref{alg:QD_profit}, which is analogous to Theorem \ref{theorem:weight_runtime} for Algorithm \ref{alg:QD_weight}.

\begin{theorem}\label{theorem:profit_runtime}
Given $k\in[0,n]$, and let $z$ be an optimal solution, within expected time $e\left(\left\lfloor Q/\gamma\right\rfloor+1\right)n^2k$, Algorithm \ref{alg:QD_profit} achieves a population $P$ such that, if $\gamma>0$ is such that $p_i/\gamma$ is integer for every item $i$ in $z$, then for any $j\in[0,k]$, there is a feasible solution $x$ where
\begin{itemize}
\item there is an integer $m$ such that $p(x^{(j)}),p(z^{(j)})\in[m\gamma,(m+1)\gamma)$,
\item $x_{(n-j)}=z_{(n-j)}$,
\item $x^{(j)}\in P$.
\end{itemize}
Moreover, for other $\gamma$ values, the first property becomes
$p(x^{(j)}),p(z^{(j)})\in[m\gamma,(m+j+1)\gamma).$
\end{theorem}
\begin{proof}
The proof proceeds similarly as that of Theorem \ref{theorem:weight_runtime}. We have the claim holds for $k=0$ since the empty set satisfies the properties for $j=0$ (i.e. $x$ and $z$ would be the same). For other $k$ values, it suffices to show that if there is such a solution $x$ for some $j\in[0,k]$: \begin{enumerate*}[label=\arabic*)]
\item any solution $y$ replacing $x^{(j)}$ in a future step must be the $j$-size prefix of another solution with the same properties, and
\item at most one bit-flip is necessary to have it also hold for $j+1$.
\end{enumerate*}
\begin{enumerate}[label=\arabic*),wide=0pt]
\item Let $y$ be the solution replacing $x^{(j)}$, we have $p(y),p(x^{(j)})\in [m\gamma,(m+1)\gamma)$ for some integer $m$, and $w(y)<w(x^{(j)})$. Let $y^*=y\vee z_{(n-j)}$, we have $p(y),p(z^{(j)})\in[m\gamma,(m+1)\gamma)$, and $w(y^*)-w(x)=w(y)-w(x^{(j)})<0$, implying $y^*$ is feasible. Therefore, $y^*$ possess the same properties as $x$. Note that this also holds for the case where $p(x^{(j)}),p(z^{(j)})\in[m\gamma,(m+j+1)\gamma)$. In this case, $p(y)\in[m\gamma,(m+j+1)\gamma)$.
\item If this also holds for $j+1$, no further step is necessary. Assuming otherwise, then $z$ contains item $j+1$, the algorithm only needs to flip the position $j+1$ in $x^{(j)}$, since $x$ and $z$ shares $(n-j-1)$-size suffix, and the $p_{j+1}$ is a multiple of $\gamma$. Since this occurs with probability at least $1/en\max_h|P_h|$, the rest follows identically, save for $\max_h|P_h|\leq\left(\left\lfloor Q/\gamma\right\rfloor+1\right)n$. If $p_{j+1}$ is a not multiple of $\gamma$, then $p(x^{(j+1)})$ may be mapped to a different profit range from $p(z^{(j+1)})$. The difference is increased by at most 1 since $p(x^{(j+1)})-p(x^{(j)})=p(z^{(j+1)})-p(z^{(j)})$, i.e. if $p(x^{(j)}),p(z^{(j)})\in[m\gamma,(m+j+1)\gamma)$ for some integer $m$, then $p(x^{(j+1)}),p(z^{(j+1)})\in[m'\gamma,(m'+j+2)\gamma)$ for some integer $m'\geq m$. Since $x$ can be replaced in a future step by another solution with a smaller profit due to rounding error, the difference can still increase, so the claim holds non-trivially.
\end{enumerate}
\end{proof}
Theorem \ref{theorem:profit_runtime} gives us the following profit guarantees of Algorithm \ref{alg:QD_profit} when $k=n$. Here $OPT$ denotes the optimal profit.
\begin{corollary}
Algorithm \ref{alg:QD_profit}, run with $\gamma>0$, within expected time $e\left(\left\lfloor Q/\gamma\right\rfloor+1\right)n^3$ obtains a feasible solution $x$ where $p(x)=OPT$ if $p_i/\gamma$ is integer for all $i=1,\ldots,n$, and $p(x)>OPT-\gamma n$ otherwise.
\end{corollary}
\begin{proof}
If $p_i/\gamma$ is integer for all $i=1,\ldots,n$, then $|p(a)-p(b)|$ is a multiple of $\gamma$ for any solutions $a$ and $b$. Since by Theorem \ref{theorem:weight_runtime}, $x$ is feasible and $p(x)>OPT-\gamma$, it must be that $p(x)=OPT$.
For the other case, Theorem \ref{theorem:weight_runtime} implies that $p(x),OPT\in[m\gamma,(m+n+1)\gamma)$ for some integer $m$. This means $p(x)>OPT-\gamma n$.
\end{proof}

Using this property, we can set up a FPRAS with an appropriate choice of $\gamma$, which is reminiscent of the classic FPTAS for KP based on DP. As a reminder, $x$ is a $(1-\epsilon)$-approximation for some $\epsilon\in(0,1)$ if $p(x)\geq(1-\epsilon)OPT$.
The following corollary is obtained from the fact that $Q\leq n\max_i\{p_i\}$, and $\max_i\{p_i\}\leq OPT$. 
\begin{corollary}
For some $\epsilon\in(0,1)$, Algorithm \ref{alg:QD_profit}, run with $\gamma=\epsilon\max_i\{p_i\}/n$, obtains a $(1-\epsilon)$-approximation within expected time $e\left(\left\lfloor n^2/\epsilon\right\rfloor+1\right)n^3$.
\end{corollary}

For comparison, the asymptotic runtime of the classic FPTAS achieving the same approximation guarantee is $O(n^2\lfloor n/\epsilon\rfloor)$ \cite{Vijay2001}.

\section{Experimental investigations}
\label{sec:experiments}

In this section, we experimentally examine the two MAP-Elite based algorithms. The experiments can be categorised in three sections. First, we illustrate the distribution of high-performing solutions in the two behavioural spaces. Second, we compare Algorithm~\ref{alg:QD_weight} and \ref{alg:QD_profit} in terms of population size and ratio in achieving the optimums over 30 independent runs. Finally, we compare between the best MAP-Elite algorithm and two baseline EAs, namely $(1+1)$EA and $(\mu+1)$EA. These baselines are selected due the the same size of offspring in each iteration. For the first round of experiments, three instances from \cite{PolyakovskiyB0MN14} are considered. There is a strong correlation between the weight and profit of each items in the first instance. The second and third instances are not correlated, while the items have similar weights in third instance. The termination criterion is set to the maximum fitness evaluations of $Cn^2$. We also set $\gamma\in\{1, 5, 25\}$. For the second and third rounds of experiments, we run algorithms on $18$ test instances from \cite{PolyakovskiyB0MN14}, and change the termination criterion to either achieving the optimal value or the maximum CPU-time of $7200$ seconds.   

\begin{figure}[t]
\centering
    \centering
    \includegraphics[width=.8\columnwidth]{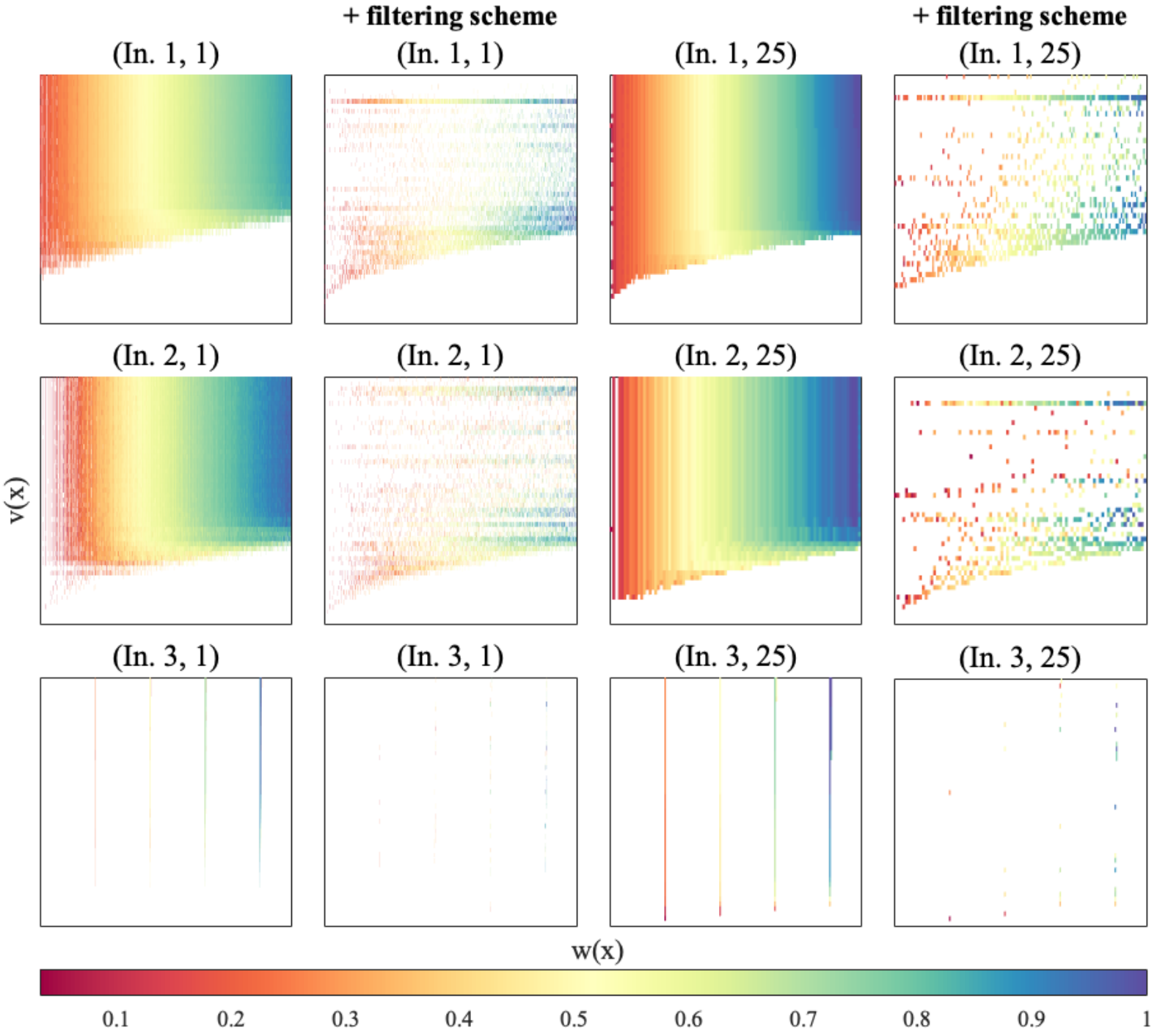}\vspace{-10pt}
    \caption{The distribution of high-performing solutions in the weight-based behavioral space. The title of sub-figures show (Ints. No, $\gamma$). Colors are scaled to OPT.}\vspace{-15pt}
    \label{fig:w(x)}
\end{figure}
\begin{figure}[t]
\centering
    \centering
    \includegraphics[width=.8\columnwidth]{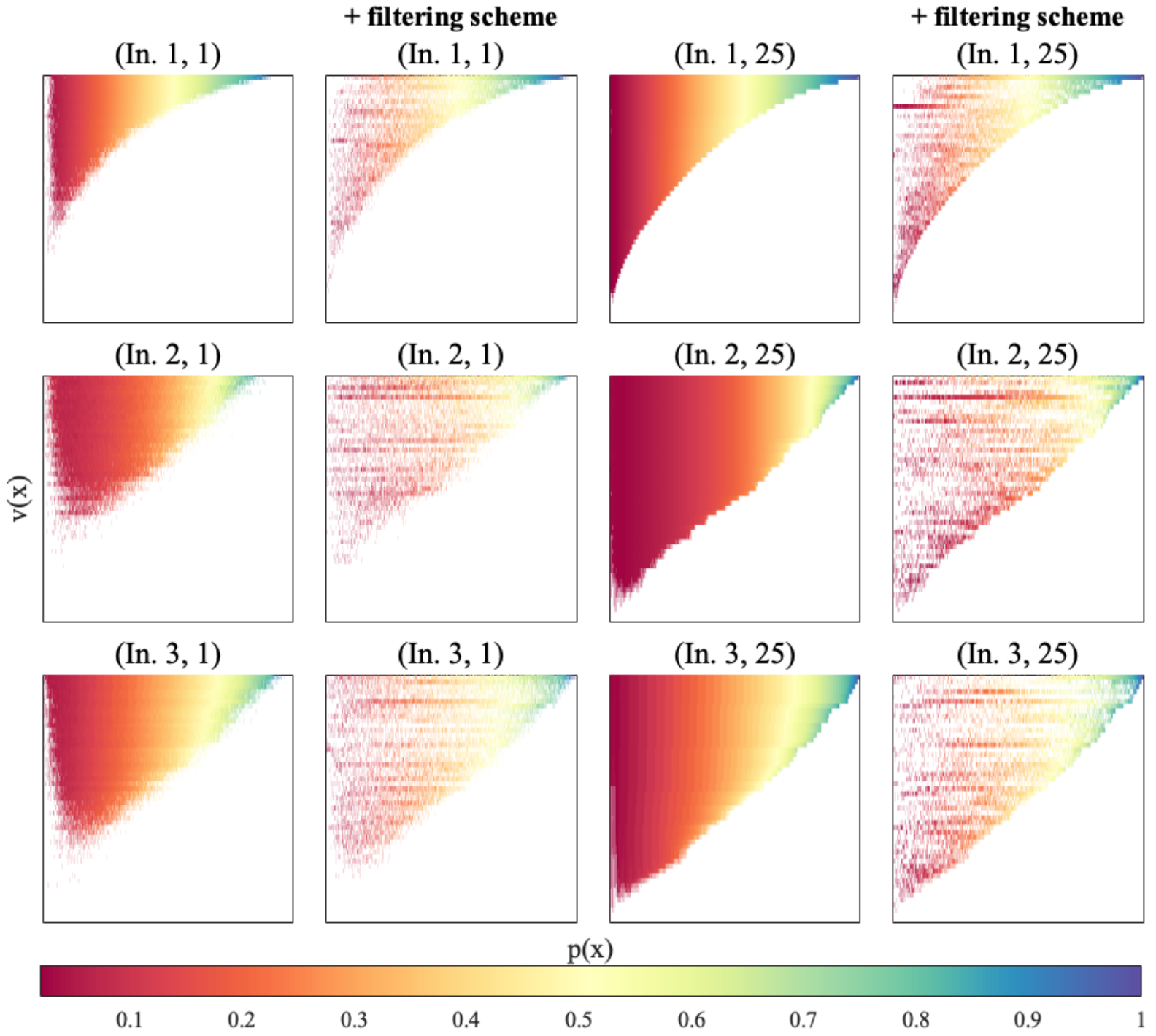}\vspace{-10pt}
    \caption{The distribution of high-performing solutions in the profit-based behavioral space. Analogous to Figure \ref{fig:w(x)}.}\vspace{-15pt}
    \label{fig:p(x)}
\end{figure}
Figure \ref{fig:w(x)} illustrates the high-performing solutions obtained by  Algorithm \ref{alg:QD_weight} in the weight-based space. As shown on the figure, the best solutions can be found the right top of the space. One can expected it since in that area of the space, solutions get to involve most items and the most of the knapsack's capacity, whereby on the left bottom of the space a few items and a small proportion of $C$ can be used. Algorithm \ref{alg:QD_weight} can successfully populates the most of the space in instance $1$, $2$, while we can see most of the space is empty in instance $3$. This is because the weights are uniformly distributed within $[1000, 1010]$, while $C$ is set to $4567$. As shown on the figure, the feasible solutions can only pick $4$ items. Figure \ref{fig:w(x)} also shows that the DP-based filtering removes many dominated solutions that contribute to convergence rate and pace of the algorithm.  

Figure \ref{fig:p(x)} shows the best-performing solutions obtained by  Algorithm \ref{alg:QD_profit} in the profit-based space. It can be observed that we can only populate the half of space by Algorithm~\ref{alg:QD_profit} or any other algorithm. To have a solution with profit of $Q$, the solutions needs to pick all items. This means that it is impossible to populate any other cells except cell $(n+1), (Q+1)$. On the contrary of the weight-based space, we can have both feasible and infeasible solutions in the profit-based space. For example, the map is well populated in instance $3$, but mostly contains infeasible solution. Figure \ref{fig:trd} depicts the trajectories of population size of Algorithm~\ref{alg:QD_weight} and \ref{alg:QD_profit}. The figure shows Algorithm~\ref{alg:QD_weight} results in significantly smaller $|P|$ than  Algorithm~\ref{alg:QD_profit}. For example, the final population size of Algorithm~\ref{alg:QD_weight} is equal to $37$ in instance $3$, where $\gamma = 25$, while it is around $9000$ for Algorithm \ref{alg:QD_profit}. This is because we can limit the first space to the promising part of it ($w(x)\leq W$), but we do not have the similar advantage for the profit-based space; the space accept the full range of possible profits ($p(x)\leq Q$). We believe this issue can cause an adverse effect on the efficiency of MAP-Elites in reaching optimality, based on theoretical observations. This is explored further in our second experiment, where we look at the actual run-time to achieve the optimum. 

\begin{figure}[t]
\centering
    \centering
    \includegraphics[width=.8\columnwidth]{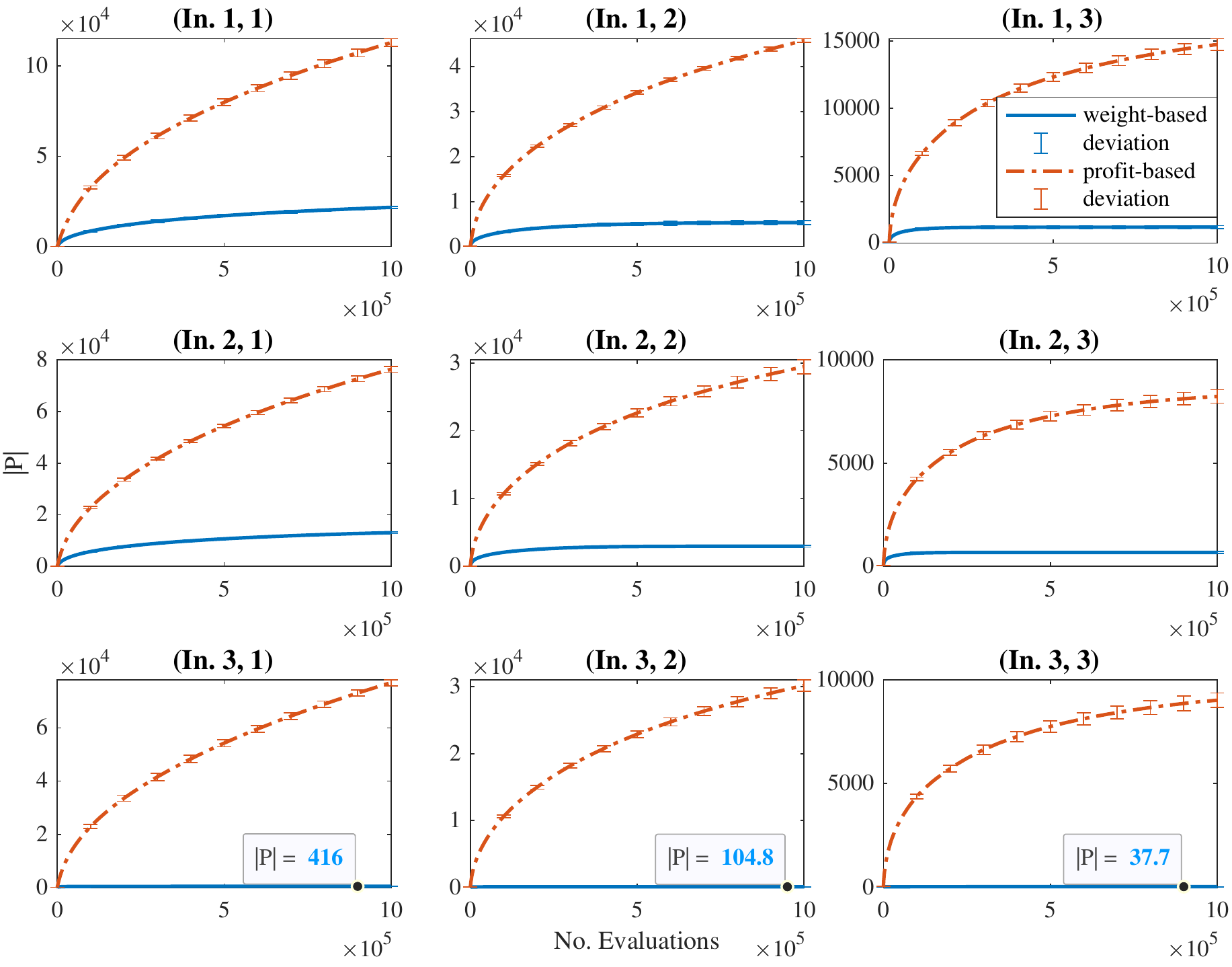}\vspace{-10pt}
    \caption{Means and standard deviations of population sizes over fitness evaluations (the filtering scheme is used).}\vspace{-15pt}
    \label{fig:trd}
\end{figure}

Table \ref{tbl:Res_W} and \ref{tbl:Res_P} show the ratio of Algorithm \ref{alg:QD_weight}  and Algorithm \ref{alg:QD_profit} in achieving the optimum for each instances in 30 independent runs, respectively. The tables also presents the mean of fitness evaluations for the algorithms to hit the optimal value or reach the limitation of CPU time. Table \ref{tbl:Res_W} shows that the ratio is $100\%$ for Algorithm \ref{alg:QD_weight} on all instances and all $\gamma \in \{1, 5, 25\}$. On the other hand, Algorithm \ref{alg:QD_profit} cannot achieve the optimums in all 30 runs, especially in large instances when $\gamma = 1$. However, increasing $\gamma$ to $25$ enables the algorithm to obtain the optimum in the most instances with exception of instance $9$. Moreover, the number of fitness evaluations required for Algorithm \ref{alg:QD_profit} is considerably higher than that of Algorithm \ref{alg:QD_weight}. We can conclude that Algorithm \ref{alg:QD_weight} is more time-efficient than Algorithm \ref{alg:QD_profit}, confirming our theoretical findings. This also suggests that the rounding errors are not detrimental to these algorithms' performances.

\begin{table*}[t!]
\centering
\caption{Number of fitness evaluations needed by Algorithm \ref{alg:QD_weight} to obtain the optimal solutions. %The ratio of achieving the optimal solution is $100\%$ in all cases. 
}
\begin{scriptsize}
\renewcommand{\tabcolsep}{3.83pt}
\renewcommand{\arraystretch}{1}
\begin{tabular}{lccc|cc|cc|cc}
\toprule
\multirow{2}{*}{Inst.}&\multirow{2}{*}{$n$}&\multirow{2}{*}{$C$}&\multirow{2}{*}{$U$}&\multicolumn{2}{c|}{$\gamma=1$}&\multicolumn{2}{c|}{$\gamma=5$}&\multicolumn{2}{c}{$\gamma=25$}\\
\cmidrule(l{2pt}r{2pt}){5-10}
&&&&mean&time&mean&time&mean&time\\
\midrule
1&50&4029&1.37e+09&1.53e+06&2.74e+01&3.76e+05&4.35e+00&1.61e+05&1.96e+00\\
2&50&2226&7.57e+08&5.32e+05&8.14e+00&1.74e+05&2.32e+00&5.86e+04&9.75e-01\\
3&50&4567&1.55e+09&2.43e+04&3.48e-01&1.12e+04&1.81e-01&5.82e+03&1.08e-01\\
4&75&5780&6.63e+09&5.30e+06&8.28e+01&1.45e+06&2.07e+01&4.12e+05&5.60e+00\\
5&75&3520&4.04e+09&3.63e+06&7.11e+01&1.15e+06&2.44e+01&3.49e+05&5.96e+00\\
6&75&6850&7.86e+09&1.17e+05&2.21e+00&4.09e+04&5.91e-01&1.44e+04&2.32e-01\\
7&100&8375&2.28e+10&2.42e+07&4.75e+02&6.57e+06&1.33e+02&6.60e+07&1.34e+03\\
8&100&4815&1.31e+10&9.56e+06&2.02e+02&2.73e+06&5.07e+01&8.66e+05&1.30e+01\\
9&100&9133&2.48e+10&6.18e+05&1.06e+01&1.78e+05&3.26e+00&5.79e+04&1.34e+00\\
10&123&10074&5.10e+10&3.56e+07&6.65e+02&9.90e+06&1.77e+02&2.55e+07&4.71e+02\\
11&123&5737&2.90e+10&2.05e+07&5.40e+02&5.12e+06&9.38e+01&1.47e+06&3.54e+01\\
12&123&11235&5.68e+10&1.45e+06&3.74e+01&3.38e+05&5.27e+00&1.21e+05&1.90e+00\\
13&151&12422&1.16e+11&5.04e+07&9.15e+02&1.48e+07&2.73e+02&7.51e+06&1.86e+02\\
14&151&6924&6.48e+10&4.27e+07&9.75e+02&1.24e+07&2.73e+02&3.35e+06&5.71e+01\\
15&151&13790&1.29e+11&3.18e+06&9.87e+01&6.73e+05&1.70e+01&2.35e+05&3.94e+00\\
\bottomrule
\end{tabular}\vspace{-15pt}
\end{scriptsize}
\label{tbl:Res_W}
\end{table*}

\begin{table*}[t!]
\centering
\caption{Number of fitness evaluations needed by Algorithm \ref{alg:QD_profit} to obtain the optimal solutions}
\begin{scriptsize}
\renewcommand{\tabcolsep}{.85pt}
\renewcommand{\arraystretch}{1}
\begin{tabular}{lcc|ccc|ccc|ccc}
\toprule
\multirow{2}{*}{Inst.}&\multirow{2}{*}{$n$}&\multirow{2}{*}{$Q$}&\multicolumn{3}{c|}{$\gamma=1$}&\multicolumn{3}{c|}{$\gamma=5$}&\multicolumn{3}{c}{$\gamma=25$}\\
\cmidrule(l{2pt}r{2pt}){4-12}
&&&mean&ratio&$U$&mean&ratio&$U$&mean&ratio&$U$\\
\midrule
1&50&53928&1.15e+07&100&1.83e+10&3.68e+06&100&e.66e+09&1.21e+06&100&7.33e+08\\
2&50&23825&5.36e+06&100&8.10e+09&1.34e+06&100&1.62e+09&4.00e+05&100&3.24e+08\\
3&50&24491&3.86e+06&100&8.32e+09&1.27e+06&100&1.66e+09&1.27e+06&100&3.33e+08\\
4&75&78483&5.07e+07&100&9.00e+10&1.50e+07&100&1.8e+10&4.03e+06&100&3.6e+09\\
5&75&37237&2.74e+07&100&4.27e+10&7.36e+06&100&8.54e+09&2.32e+06&100&1.71e+09\\
6&75&38724&1.93e+07&100&4.44e+10&5.63e+06&100&8.88e+09&2.95e+06&100&1.78e+09\\
7&100&112635&2.24e+08&97&3.06e+11&6.67e+07&100&6.12e+10&1.72e+07&100&1.22e+10\\
8&100&48042&6.76e+07&100&1.31e+11&1.82e+07&100&2.61e+10&5.03e+06&100&5.22e+09\\
9&100&52967&7.99e+07&100&1.44e+11&2.86e+07&100&2.88e+10&1.18e+08&87&5.76e+09\\
10&123&135522&3.35e+08&87&6.86e+11&1.05e+08&100&1.37e+11&7.34e+07&100&2.74e+10\\
11&123&57554&1.47e+08&100&2.91e+11&3.58e+07&100&5.82e+10&8.87e+06&100&1.16e+10\\
12&123&63116&1.71e+08&97&3.19e+11&8.45e+07&100&5.38e+10&2.66e+07&100&1.28e+10\\
13&151&166842&3.81e+08&13&1.56e+12&1.39e+08&100&3.12e+11&5.69e+07&100&6.25e+10\\
14&151&70276&3.15e+08&77&6.58e+11&8.86e+07&100&1.32e+11&1.96e+07&100&2.63e+10\\
15&151&76171&2.64e+08&90&7.13e+11&9.58e+07&100&1.42e+11&3.16e+07&100&2.85e+10\\
16&194&227046&2.94e+08&0&4.51e+12&3.33e+08&23&9.01e+11&1.30e+08&100&1.8e+11\\
17&194&92610&3.43e+08&0&1.84e+12&2.22e+08&97&3.68e+11&5.80e+07&100&7.35e+10\\
18&194&97037&3.55e+08&0&1.93e+12&2.13e+08&87&3.85e+11&9.98e+07&100&7.7e+10\\
\bottomrule
\end{tabular}\vspace{-15pt}
\end{scriptsize}
\label{tbl:Res_P}
\end{table*}

For the last round of the experiments, we compare Algorithm \ref{alg:QD_weight} to two well-known EAs in the literature, $(1+1)$EA and $(\mu+1)$EA. Table \ref{tbl:Res_EAs} presents the ratio of the three algorithms in achieving the optimum and the mean of fitness evaluations required for them to reach the optimum. As shown on the table, the performances of $(1+1)$EA and $(\mu+1)$EA deteriorate on the strongly correlated instances. It seems that $(1+1)$EA and $(\mu+1)$EA are prone to get stuck in local optima, especially in instances with a strong weights-profits correlation. On the other hand, the MAP-Elite algorithm performs equally good in all instances through the diversity of solutions. Moreover, the mean of its runtime is significantly less in the half of instances although the population size of Algorithm \ref{alg:QD_weight} can be significantly higher that the other two EAs.        
\begin{table*}[t!]
\centering
\caption{Comparison %among QD, $(1+1)$EA, and $(\mu+1)$EA 
in ratio, number of required fitness evaluations and required CPU time for hitting the optimal value in 30 independent runs.}
\begin{scriptsize}
\renewcommand{\tabcolsep}{.7pt}
\renewcommand{\arraystretch}{1}
\begin{tabular}{lc|cccc|cccc|cccc}
\toprule
\multirow{2}{*}{Inst.}&\multirow{2}{*}{$n$}&\multicolumn{4}{c|}{QD}&\multicolumn{4}{c|}{$(1+1)$EA}&\multicolumn{4}{c}{$(\mu+1)$EA}\\
\cmidrule(l{2pt}r{2pt}){3-6}
\cmidrule(l{2pt}r{2pt}){7-10}
\cmidrule(l{2pt}r{2pt}){11-14}
% \cmidrule(l{2pt}r{2pt}){8-9}
% \cmidrule(l{2pt}r{2pt}){10-11}
% \cmidrule(l{2pt}r{2pt}){12-13}
&& ratio & mean & time & Stat & ratio  & mean & time & Stat & ratio  & mean & time & Stat\\
\midrule
1&50&\hl{100}&1.53e+06&2.74e+01&$2^-3^-$&40&1.32e+09&4.49e+03&$1^+3^*$&40&4.59e+08&4.92e+03&$1^+2^*$\\
2&50&100&5.32e+05&8.14e+00&$2^*3^*$&100&5.30e+05&2.02e+00&$1^*3^*$&100&6.10e+05&6.21e+00&$1^*2^*$\\
3&50&100&2.43e+04&3.48e-01&$2^+3^+$&100&1.01e+04&4.38e-02&$1^-3^*$&100&1.21e+04&1.50e-01&$1^-2^*$\\
4&75&100&5.30e+06&8.28e+01&$2^*3^*$&97&1.20e+08&5.70e+02&$1^*3^*$&100&3.46e+07&6.90e+02&$1^*2^*$\\
5&75&100&3.63e+06&7.11e+01&$2^*3^*$&100&6.34e+07&3.44e+02&$1^*3^*$&100&3.16e+07&4.44e+02&$1^*2^*$\\
6&75&100&1.17e+05&2.21e+00&$2^+3^+$&100&1.30e+04&8.98e-02&$1^-3^-$&100&2.14e+04&3.20e-01&$1^-2^+$\\
7&100&\hl{100}&2.42e+07&4.75e+02&$2^-3^-$&63&5.72e+08&3.26e+03&$1^+3^*$&43&2.51e+08&4.92e+03&$1^+2^*$\\
8&100&100&9.56e+06&2.02e+02&$2^+3^+$&100&2.33e+06&1.46e+01&$1^-3^*$&100&3.56e+06&5.81e+01&$1^-2^*$\\
9&100&100&6.18e+05&1.06e+01&$2^+3^+$&100&3.72e+04&2.24e-01&$1^-3^*$&100&5.03e+04&8.84e-01&$1^-2^*$\\
10&123&\hl{100}&3.56e+07&6.65e+02&$2^-3^-$&77&3.90e+08&2.44e+03&$1^+3^*$&47&2.34e+08&4.70e+03&$1^+2^*$\\
11&123&\hl{100}&2.05e+07&5.40e+02&$2^*3^+$&97&1.38e+08&1.13e+03&$1^*3^*$&87&6.10e+07&1.41e+03&$1^-2^*$\\
12&123&100&1.45e+06&3.74e+01&$2^+3^+$&100&6.55e+04&5.24e-01&$1^-3^*$&100&6.71e+04&1.34e+00&$1^-2^*$\\
13&151&\hl{100}&5.04e+07&9.15e+02&$2^*3^*$&97&1.25e+08&1.07e+03&$1^*3^*$&87&8.65e+07&2.14e+03&$1^*2^*$\\
14&151&100&4.27e+07&9.75e+02&$2^+3^+$&100&1.20e+07&1.06e+02&$1^-3^*$&100&1.10e+07&3.33e+02&$1^-2^*$\\
15&151&100&3.18e+06&9.87e+01&$2^+3^+$&100&1.17e+05&1.09e+00&$1^-3^*$&100&1.09e+05&2.59e+00&$1^-2^*$\\
16&194&\hl{100}&1.58e+08&4.22e+03&$2^-3^-$&57&4.99e+08&4.21e+03&$1^+3^*$&47&2.34e+08&5.47e+03&$1^+2^*$\\
17&194&\hl{100}&1.18e+08&2.25e+03&$2^-3^*$&57&4.29e+08&3.91e+03&$1^+3^*$&40&2.07e+08&4.87e+03&$1^*2^*$\\
18&194&100&7.76e+06&1.50e+02&$2^+3^+$&100&1.17e+05&1.42e+00&$1^-3^*$&100&1.37e+05&4.71e+00&$1^-2^*$\\
\bottomrule
\end{tabular}\vspace{-15pt}
\end{scriptsize}
\label{tbl:Res_EAs}
\end{table*}
\section{Conclusions}
In this study, we examined the capability of QD approaches and in particular, MAP-Elite in solving knapsack problem. We defined two behavioural spaces inspired by the classic DP approaches, and two corresponding MAP-Elite-based algorithms operating on these spaces. We established that they imitate the exact DP approach, and one of them behaves similarly to the classic FPTAS for KP under a specific parameter setting, making it a FPRAS. We then compared the runtime of the algorithms empirically on instances of various properties related to their hardness, and found that the MAP-Elite selection mechanism significantly boosts efficiency of EAs in solving KP in terms of convergence ratio, especially in hard instances. Inspecting the behavioural spaces and population sizes reveals that smaller populations correlate to faster optimisation, demonstrating a well-known trade-off between optimisation and exploring behavioural spaces.

It is an open question to which extent MAP-Elites can simulate DP-like behaviours in other problems with recursive subproblem structures. Moreover, it might be possible to make such approaches outperform DP via better controls of behavioural space exploration, combined with more powerful variation operators.

\section*{Acknowledgements}
This work was supported by the Australian Research Council through grants DP190103894 and FT200100536.
%\newpage
\bibliographystyle{splncs04}
\bibliography{Main}
\end{document}